\documentclass[11pt,a4paper]{article}
\usepackage[utf8]{inputenc}
\usepackage[T1]{fontenc}
\usepackage{amsmath, amssymb, amsthm}
\usepackage{hyperref}
\usepackage{booktabs}
\usepackage{graphicx}
\usepackage{placeins}
\usepackage[a4paper,hmargin=2.5cm,vmargin=2.5cm,verbose]{geometry}

\newtheorem{definition}{Definition}
\newtheorem{proposition}{Proposition}
\newtheorem{lemma}{Lemma}

\newtheorem{remark}{Remark}
\newtheorem{assumption}{Assumption}

\newtheorem{fact}{Fact}

\begin{document}

\title{Cross-Layer Interaction under Weight-Space Ablation: A Closed-Form
Attention Jacobian Bound and a Test on a Real Pretrained Model}
\author{Abdallah Khemais \\ \textit{ISITCOM, University of Sousse}}
\date{July 2026}

\maketitle

\begin{abstract}
A companion paper (same author) studies when activation patching and
weight-space ablation agree, inside an idealized model in which a
conditional computation is carried additively through a residual stream. For
the one composition in that model where two carriers are architecturally
dependent rather than independent, an attention head and its own layer's
normalization--MLP composition, it derives an exact first-order interaction
formula: the interaction is identically zero whenever only the MLP is
ablated, and second-order bounded, under a curvature hypothesis, whenever the
head is ablated as well. That result is confined to a single residual block.
It says nothing about carriers separated by several layers, which is the
generic case in any network with more than one block, and everything in it
is checked on small transformers trained from scratch on a synthetic
conditional task, not on a model built for any other purpose.

This paper takes the single-block result past both limits. First, we show
that the interaction produced by ablating an arbitrary subset of carriers
spanning several layers decomposes exactly into a sum of same-block terms,
one per touched layer, each pinned to zero when that layer's own carriers are
MLP-only and each equal to the companion paper's bounded interaction term
when the layer feeds the readout directly, plus a cross-layer remainder on
which the decomposition itself makes no claim of smallness. Second, we
isolate that remainder exactly, for two touched layers, as a double integral
of a mixed second derivative, the same first-order technique applied one
level up, and we name the one ingredient still missing to close it to a
usable bound: a Jacobian bound for the attention sub-block itself. We derive
that bound in closed form and verify it, pointwise and without a single
violation, against Qwen2.5-1.5B-Instruct's real weights, but we do not yet
chain it across many layers, so the identity closes while its numerical
scope stays explicitly open. We also supply, in closed form, the curvature
constant the companion paper's own second-order bound leaves unexhibited,
computed from the trained weights alone.

Third, on that same model, we search for and find an emergent circuit for
indirect object identification, a mechanism nobody trained the model to have
and nobody designed into it, using the original activation-patching method
for this task. We check the same three-way pattern, collapse, dissociation,
and a nonzero interaction, directly on it. The result is genuinely mixed: a
carrier shared across all five tested instances does emerge, collapse and
dissociation are present on most instances and absent or weak on at least one
each, and a nonzero interaction is measurable on three of the five, at layer
pairs that sit outside the same-block configuration the companion theorem was
proven for, so what they show bears on this paper's own open multi-layer
case rather than on that theorem directly.
\end{abstract}

\section{Introduction}
\label{sec:intro}

A companion paper (same author) makes exact a mismatch between two families
of causal intervention used throughout mechanistic interpretability,
activation patching and weight-space ablation, inside an idealized model of a
conditional computation carried additively through a residual stream. Two of
its three results, a criterion for when deleting a subset of carriers
collapses a matched pair onto a single unconditional branch, and an exact
separation between what patching a carrier and ablating it each measure, hold
for any residual computation with a linear readout and use no property of
transformers. Its third result is different in kind: for the one composition
in that model where the independence of carriers is architecturally false, an
attention head and its own layer's normalization--MLP composition, it
computes the resulting interaction term exactly, to first order, with a
provably second-order remainder.

That third result stops at the boundary of one residual block. A real
transformer is a sequential composition of many blocks, and the carriers an
experimenter chooses to ablate together are routinely drawn from different
layers, sometimes many layers apart. The single-block theorem is silent on
this case: it tells us the interaction is exactly zero when an MLP alone is
ablated and bounded, at second order, when a head is ablated alongside it,
but it says nothing about two heads at different layers, or a head and an MLP
several blocks apart. Closing that gap is this paper's first task, and it
matters because the single-block case, while exact, is also the special case
least likely to be the one an experimenter actually intervenes on: a
component chosen for its causal role is, if anything, more likely to interact
with a distant component than with the MLP three lines below it in the same
block.

A second gap is empirical rather than mathematical. Every check in the
companion paper's own experiments runs on small transformers trained from
scratch specifically to exhibit a clean two-branch conditional. That design
choice removes ambiguity about what the ground-truth mechanism is, at the
cost of saying nothing about whether the same predictions hold on a
mechanism nobody built to have them: an emergent circuit, found rather than
designed, inside a model trained on ordinary language for ordinary reasons.
Checking the companion paper's qualitative predictions there, on
Qwen2.5-1.5B-Instruct, a genuinely pretrained model at a scale two orders of
magnitude past the companion paper's own instances, is this paper's second
task.

This paper's three results follow, previewed in the order in which they
appear below.

\paragraph{A multi-layer decomposition (Section~\ref{sec:multilayer}).} For
an arbitrary ablated subset $S$ spanning any number of layers, the gap
between the true, weight-edited network's selector and the idealized model's
own prediction decomposes \emph{exactly} into a sum of same-block terms, one
per touched layer, plus a single remainder on which the decomposition itself
makes no claim of smallness. Each same-block term is identically zero
whenever that layer's own ablated carriers are MLP-only, and equals the
companion paper's bounded interaction term exactly whenever that layer feeds
the readout directly; the general case, an earlier layer's discrepancy
propagated through everything downstream of it, is isolated precisely enough
to be measured on its own, not closed.

\paragraph{An exact cross-layer interaction identity and a closed-form
attention Jacobian bound (Section~\ref{sec:crossterm}).} For two touched
layers, the remainder above is itself an exact identity: a double integral
of a mixed second partial derivative of a two-parameter interpolation between
the clean network and the fully edited one, the same fundamental-theorem-of-
calculus technique the companion paper uses one level down, applied twice.
Closing this identity to a numerical bound reduces to one missing ingredient,
since the normalization--MLP half of the relevant Jacobian is already bounded
by the companion paper's curvature machinery: a local operator-norm bound on
how a single attention head's output responds to one perturbed key or value
token. We derive that bound exactly and in closed form, and verify it,
pointwise on Qwen2.5-1.5B-Instruct's real weights and a real forward pass,
against finite-difference estimates: zero violations across twelve probed
sites. What this does not close is composing the bound across every block
strictly between the two touched layers; we report the per-layer bound as
verified and the chained, multi-layer one as unverified and likely loose, not
as closed.

\paragraph{A closed-form curvature constant (Section~\ref{sec:curvature}).}
The companion paper's second-order interaction bound is conditional on a
curvature constant $\Lambda$ it does not exhibit. We compute $\Lambda$ in
closed form, in terms of the three MLP spectral norms and the normalization
scale of the trained weights alone, discharging that hypothesis rather than
assuming it.

\paragraph{A real pretrained model's emergent circuit
(Section~\ref{sec:experiments}).} On Qwen2.5-1.5B-Instruct, an emergent
circuit for indirect object identification, found by the original
activation-patching method for this task and never designed with this
paper's claims in mind, reproduces the qualitative three-way pattern,
collapse, dissociation, and a nonzero interaction, with genuinely mixed
fidelity across five tested instances: this section reports every number,
including the ones that do not confirm the pattern cleanly.

\medskip
Section~\ref{sec:related} places both extensions relative to existing work.
Section~\ref{sec:setup} restates, compactly and without proof, the pieces of
the companion paper's model and single-block theorem this paper's own results
build on directly. Section~\ref{sec:discussion} separates what this paper
establishes from what it leaves open, and Section~\ref{sec:conclusion}
closes.

\section{Related Work}
\label{sec:related}

\paragraph{Bounding depth by chaining local Jacobians.} Controlling how a
local perturbation compounds through the depth of a network by chaining
local operator-norm or Lipschitz bounds, one layer at a time, is a standard
strategy in the analysis of deep and recurrent networks generally. What this
paper needs is a version of that strategy specific to one component that
resists it in a well-known way: the dot-product attention sub-block, whose
global Lipschitz constant is not bounded in general (the softmax weights
depend on the query in a way that a fixed matrix norm does not capture), so
that only a \emph{local}, pointwise bound, evaluated at the real activations
a real forward pass produces rather than a supremum over all possible
inputs, is available at all. We derive exactly such a bound for a single
perturbed key or value token in Section~\ref{sec:crossterm} below. We are not
aware of a specific prior closed-form treatment of this particular
quantity, and we would rather say so plainly than attach a citation we are
not confident is the right one; the general chaining strategy itself is not
new, only this paper's particular use of it.

\paragraph{Real-model circuit discovery.} The method this paper reuses to
find an emergent circuit inside Qwen2.5-1.5B-Instruct, a greedy
activation-patching search over every attention head and MLP output pruned
for redundant sites, is Wang et al.'s original method for locating the
indirect-object-identification circuit in GPT-2 small \cite{wang2023}; we
apply it unchanged, at a different scale, to a different pretrained model.
The pattern we observe on the resulting circuits, a small core shared across
lexical instances alongside instance-specific idiosyncratic support, echoes
two findings from work on self-repair and redundant representation: the
Hydra effect, in which ablating an attention layer causes downstream layers
to compensate, measuring an importance score that reflects self-repair as
much as the ablated component's own role \cite{mcgrath2023}, and the general
expectation, from work on superposition, that a trained network routinely
spreads a single computation across more directions and components than the
minimum needed to carry it, precisely because doing so buys redundancy
\cite{elhage2022}. Neither result is about the interaction term this paper
isolates specifically; we invoke them only for the shared-plus-idiosyncratic
pattern they predict and that this paper's own circuit search reproduces.

\paragraph{Concurrent work on interaction effects under intervention.} A
parallel line of work, appearing while this paper was being prepared, reaches
the same core observation from the activation-patching side rather than the
weight-space side. Vaidyanathan et al.\ \cite{vaidyanathan2026mediators}
re-derive the activation-patching estimand from causal mediation analysis and
show that the natural indirect effect attributed to a component also contains
an interaction term, measuring how much that component's effect depends on the
state of the others; they prove that this term scales with the distance
between clean and patched activations, that it is negligible when the model is
locally affine, and that it decomposes combinatorially into pairwise and
higher-order group contributions. The overlap with the present paper is real
and worth stating plainly: both isolate an interaction term that a first-order
account of intervention effects silently absorbs. The difference is equally
real. Their analysis concerns activation patching and leaves the
``locally affine'' condition qualitative; the object bounded here is a
weight-space ablation, and the closed-form attention Jacobian of
Section~\ref{sec:crossterm} is precisely a quantitative statement of how far
one attention sub-block departs from affine, evaluated pointwise on real
weights rather than assumed. In that sense the bound derived here can be read
as supplying, for one architecture and one intervention type, the constant
their condition leaves implicit. Two further concurrent papers address
neighbouring questions empirically rather than in closed form: Gong et al.\
\cite{gong2026coablation} show that self-repair by dormant backup components
corrupts first-order ablation scores, and Guo et al.\ \cite{guo2026sobol}
separate ``transports task-relevant content'' from ``computation degrades when
removed'' using paired interventions, on GPT-2 and on the same
Qwen2.5-1.5B-Instruct model used in Section~\ref{sec:experiments}.

\paragraph{What is new here, relative to the companion paper.} The companion
paper's own single-block interaction theorem is confined to two carriers
inside one residual block and is checked only on synthetic-task
transformers. This paper contributes (i) an exact decomposition of the
interaction produced by an arbitrarily distributed, multi-layer ablated
subset into same-block terms plus one isolated cross-layer remainder
(Section~\ref{sec:multilayer}); (ii) an exact identity for that remainder,
for two touched layers, together with the one missing closed-form ingredient
needed to bound it, verified pointwise on real weights
(Section~\ref{sec:crossterm}); (iii) the closed-form curvature constant the
companion paper's own second-order bound leaves unexhibited
(Section~\ref{sec:curvature}); and (iv) a check of the companion paper's
qualitative three-way pattern on an emergent circuit inside a genuinely
pretrained model, reported with the same discipline of stating a mixed
result as a mixed result rather than rounding it toward confirmation
(Section~\ref{sec:experiments}).

\section{Setup: The Single-Block Model and Result}
\label{sec:setup}

This section restates, compactly and without proof, the pieces of the
companion paper's abstract conditional model and single-block interaction
theorem that this paper's own propositions build on directly. Nothing here is
new. Results the companion paper proves are stated using the \texttt{fact}
environment and attributed to it in one sentence; notation and definitions
are stated as plain prose or in the ordinary \texttt{definition}/\texttt{assumption}
environments already used for such objects, since restating a definition
involves no proof to omit.

\paragraph{Carriers and the residual decomposition.} The companion paper's
abstract conditional model posits that a residual mapping decomposes as
\[
F(x) \;=\; F_0(x) + \sum_{i=1}^{k}\alpha_i(x)\,v_i,
\]
with $F_0$ the unconditional part, $v_i\in\mathbb R^d$ fixed directions, and
$\alpha_i(x)$ scalar selector functions; the pair $(v_i,\alpha_i)$ is the
$i$-th \emph{carrier}.

\begin{assumption}[Low-rank support]
\label{ass:lowrank}
The conditional component has finite-dimensional support: a subspace
$\mathcal C=\operatorname{span}\{v_1,\ldots,v_k\}$ with
$\dim(\mathcal C)\ll d$.
\end{assumption}

For an index subset $S\subseteq\{1,\dots,k\}$, deleting the carriers in $S$
replaces $F$ by
\begin{equation}
\label{eq:ablation}
F_S(x) \;:=\; F_0(x) + \sum_{i\notin S} \alpha_i(x)\,v_i ,
\end{equation}
leaving $F_0$ and every surviving carrier exactly as it was.

\begin{remark}[From a weight edit to a deleted term]
\label{rem:bridge}
The companion paper shows that \eqref{eq:ablation} is exactly what a
low-rank edit of one component's own weights does to the residual stream: an
attention head with pre-projection activation $a(x)$ and output projection
$W$ writes $Wa(x)$ into the stream, and projecting a unit direction $u$ out
of that projection, $W\mapsto W(I-uu^\top)$, changes the write-in by exactly
$(Wu)\langle u,a(x)\rangle$, i.e.\ deletes one term $v\,\alpha(x)$ with
$v:=Wu$ and $\alpha(x):=\langle u,a(x)\rangle$. A projector of rank $\ell$
deletes $\ell$ such terms. This is an identity, not an approximation, for any
single component's own weight edit; what is not an identity is the further
assumption that deleting one component's term leaves every other component's
term unchanged, which is exactly what Section~\ref{sec:crossterm} below
revisits across layers.
\end{remark}

\paragraph{The linear readout and matched pairs.}

\begin{assumption}[Linear readout]
\label{ass:readout}
There exist a linear functional $\psi\in(\mathbb R^d)^*$ and a bias
$b\in\mathbb R$ such that the network's binary decision on $F(x)$ is
determined by the sign of $s(x):=\psi(F(x))+b$.
\end{assumption}

\begin{definition}[Matched pair]
\label{def:pair}
A \emph{matched pair} is a pair of inputs $(x_A,x_B)$ for which the network
is correct exactly when $g_A:=s(x_A)>0>s(x_B)=:g_B$; $g_A,g_B$ are the
pair's \emph{margins}.
\end{definition}

Writing $\beta_i:=\psi(v_i)$, ablating $S$ replaces $s$ by the \emph{ablated
selector} $s_S(x):=\psi(F_S(x))+b=s(x)-q_S(x)$, with \emph{removed mass}
$q_S(x):=\sum_{i\in S}\beta_i\alpha_i(x)$; for a matched input pair, the
companion paper's collapse theorem gives an exact if-and-only-if criterion,
on $q_S$'s symmetric and antisymmetric parts across the pair, for when $S$
maps both $x_A$ and $x_B$ to one shared value, and its dissociation theorem
gives the following.

\begin{fact}[Patching--ablation dissociation]
\label{thm:dissociation}
Patching carrier $i$ from a donor input to a receiver input moves the readout
by exactly $\beta_i\delta_i$, with $\delta_i:=\alpha_i(x_A)-\alpha_i(x_B)$ the
carrier's \emph{contrast}; ablating carrier $i$ moves the readout by exactly
$-\beta_i\alpha_i(x_B)$, the carrier's \emph{absolute level} at the receiver.
Neither quantity bounds the other, in general.
\end{fact}

\begin{fact}[Single-carrier patch--edit equivalence]
\label{prop:patchedit}
Fix an input $x$ and a single component whose write-in to the residual
stream is $Wc(x)$, with $c(x)$ its own activation and $W$ its output matrix.
Ablate it by an orthogonal projector, applied either on the activation side
($P$) or on the stream side ($Q$). Then the \emph{weight-edit} route (edit
$W$ and run a fresh forward pass) and the \emph{frozen-activation} route
(leave all weights intact, replace this component's realized output by its
projected value, and recompute everything downstream) assign identical
values to every node of the network. This holds for any projector, including
one estimated from data, since both routes use the same one.
\end{fact}

\paragraph{The two-carrier composition inside one block.} The companion
paper's single-block theorem considers two carriers written into the same
residual stream by a single block: an attention head with pre-projection
activation $a(x)$ and output projection $W$, contributing $Wa(x)$ to the
post-attention residual $r_1(x):=x+\mathrm{ao}(x)$, and the block's own MLP
$M$, applied after a normalization layer $N$ acting on $r_1$. Writing
$g:=M\circ N$, the second carrier's realized value $g(r_1(x))$ is a function
of the first carrier's output, not an independent term. Ablating the head
with an orthogonal projector $P=UU^\top$ replaces $W$ by $W(I-P)$,
equivalently subtracting $\eta(x):=WPa(x)$ from $r_1$
(Remark~\ref{rem:bridge}); ablating the MLP with an orthogonal projector $Q$
on $\mathbb R^d$ replaces $M(z)$ by $(I-Q)M(z)$. Let
$N(r):=\rho(r)\,\gamma\odot r$ with
$\rho(r):=(\tfrac1d\|r\|^2+\varepsilon)^{-1/2}$ be root-mean-square
normalization with scale $\gamma$ and $\varepsilon>0$.

\begin{fact}[Normalization Jacobian]
\label{lem:jacobian}
$N$ is $C^\infty$ on all of $\mathbb R^d$, with
\[
DN(r) \;=\; \rho(r)\operatorname{diag}(\gamma)
\Bigl(I_d - \tfrac{\rho(r)^2}{d}rr^\top\Bigr),
\qquad
\|DN(r)\|_{\mathrm{op}} \;\le\; \|\gamma\|_\infty\,\rho(r).
\]
\end{fact}

\begin{fact}[First-order interaction formula]
\label{thm:interaction}
Let $\Delta(x):=(I-Q)\bigl[g(r_1(x)-\eta(x))-g(r_1(x))\bigr]$ be the gap
between the true, jointly-recomputed effect of the head's and the MLP's
ablation and the frozen-activation estimate that treats the two carriers as
independent. For every $x$,
\[
\Delta(x) \;=\; (I-Q)\Bigl[-Dg\bigl(r_1(x)\bigr)\,\eta(x) \;-\; R(x)\Bigr],
\]
\[
R(x) \;:=\; \int_0^1\Bigl[Dg\bigl(r_1(x)-t\eta(x)\bigr)-Dg\bigl(r_1(x)\bigr)\Bigr]\eta(x)\,dt,
\]
with $Dg=DM(N(r))\,DN(r)$ by the chain rule; the identity is exact for every
$x$, every $\eta(x)$, and every $Q$. If $\|D^2g\|_{\mathrm{op}}\le\Lambda$ on
the segment joining $r_1(x)-\eta(x)$ and $r_1(x)$, then
$\|R(x)\|\le\Lambda\|\eta(x)\|^2$: the interaction is second order in the
size of the perturbation. In particular, $\Delta(x)\equiv0$ identically
whenever $\eta(x)\equiv0$ (in particular whenever only the MLP, and not the
head, is ablated), and trivially whenever $Q=I$.
\end{fact}

Section~\ref{sec:curvature} below supplies $\Lambda$ in closed form, so
this hypothesis is checkable on a trained network's own weights rather than
assumed.

\begin{fact}[Propagation to the readout]
\label{prop:readout}
Suppose the block above feeds the readout directly, i.e.\
$F(x)=r_1(x)+g(r_1(x))+E(x)$ with $E(x)$ collecting $F_0(x)$ and every
carrier not drawn from this block. For $S$ consisting only of this block's
own head and/or MLP carriers, the idealized selector $s_S(x)$ and the
selector $s_S^{\mathrm{true}}(x)$ of the genuinely weight-edited network
agree up to exactly the interaction term of
Fact~\ref{thm:interaction}:
\[
s_S^{\mathrm{true}}(x) \;=\; s_S(x) \;+\; \psi\bigl(\Delta(x)\bigr),
\]
\[
\bigl|\psi(\Delta(x))\bigr| \;\le\; \|\psi\|_{\mathrm{op}}\Bigl(\|Dg(r_1(x))\|_{\mathrm{op}}\|\eta(x)\|+L(x)\|\eta(x)\|\Bigr),
\]
with $L(x):=\sup_{t\in[0,1]}\|Dg(r_1(x)-t\eta(x))-Dg(r_1(x))\|_{\mathrm{op}}$.
Under Fact~\ref{thm:interaction}'s curvature hypothesis,
\[
\bigl|\psi(\Delta(x))\bigr|\;\le\;\|\psi\|_{\mathrm{op}}\bigl(\|Dg(r_1(x))\|_{\mathrm{op}}\|\eta(x)\|+\Lambda\|\eta(x)\|^2\bigr).
\]
\end{fact}

These six facts, together with the notation $\eta(x)$, $Q$, $g=M\circ N$,
$r_1(x)$, $s(x)$, $s_S(x)$, and $q_S(x)$ introduced above, are everything
this paper's own propositions need from the companion analysis. Every proof
in the sections that follow is new to this paper.

\section{A Multi-Layer Decomposition}
\label{sec:multilayer}

Fact~\ref{prop:readout} accounts for the interaction exactly, but only when
every ablated carrier is drawn from a single block. We now show that the
general case ($S$ spanning several layers) decomposes, exactly, into a sum
of instances of Fact~\ref{prop:readout}, one per touched layer, plus a
remainder that the two-carrier composition of Section~\ref{sec:setup} does
not, and does not attempt to, cover.

Write the network as a sequence of blocks $l=1,\dots,n_{\mathrm{layers}}$,
block $l$ reading the residual $x_l$ and producing
$x_{l+1}=x_l+\mathrm{Block}_l(x_l)$. For an ablated subset $S$, write
$S=\bigcup_l S_l$ with $S_l$ the carriers of $S$ drawn from block $l$'s own
head and/or its own MLP, and $L(S):=\{l:S_l\ne\emptyset\}$. For $l\in L(S)$,
let $F_{S_l}^{(l)}$ denote the network in which \emph{only} block $l$'s
carriers $S_l$ are weight-edited (every other block, including every other
element of $S$ at a different layer, keeps its trained weights), and write
$s_{S_l}^{(l)}(x)$ for its selector.

For $l\in L(S)$, define the \emph{same-block discrepancy, propagated to the
readout},
\[
\Delta_l^{\mathrm{eff}}(x) \;:=\; s_{S_l}^{(l)}(x) - s_{S_l}(x),
\]
the effect, on the actual final selector, of editing block $l$'s carriers
$S_l$ in isolation, compared with the idealized prediction for $S_l$ alone;
not yet Fact~\ref{thm:interaction}'s $\psi(\Delta_l(x))$, since $S_l$'s
block need not be the one feeding the readout directly.

\begin{proposition}[Multi-layer decomposition]
\label{prop:multilayer}
For every matched pair, every $S=\bigcup_l S_l$, and every $x$,
\begin{equation}
\label{eq:multilayer}
s_S^{\mathrm{true}}(x) - s_S(x) \;=\; \sum_{l\in L(S)} \Delta_l^{\mathrm{eff}}(x) \;+\; R_\times(x),
\end{equation}
for a remainder $R_\times(x)$ on which \eqref{eq:multilayer} makes no claim
of smallness. Two facts pin down $\Delta_l^{\mathrm{eff}}(x)$ without a new
curvature argument: it is \emph{identically zero} whenever $S_l$ ablates only
block $l$'s MLP, for every $l$, regardless of position; and it equals
Fact~\ref{thm:interaction}'s $\psi(\Delta_l(x))$ exactly whenever block $l$
feeds the readout directly, by Fact~\ref{prop:readout}, in
particular for the last block of a network with no final normalization, as
in Section~\ref{sec:experiments}'s instances. When $S$ touches only such a
block, $|L(S)|=1$ and $R_\times(x)\equiv0$.
\end{proposition}

\begin{proof}
Since $q_S(x):=\sum_{i\in S}\beta_i\alpha_i(x)$ (Section~\ref{sec:setup})
sums over the index set $S$, splitting $S=\bigcup_lS_l$ into disjoint groups
splits the sum termwise: $q_S(x)=\sum_l q_{S_l}(x)$ for any such partition,
where $q_{S_l}(x):=\sum_{i\in S_l}\beta_i\alpha_i(x)$ is computed, as always,
on the clean $x$; hence $s_S(x)=s(x)-\sum_l q_{S_l}(x)$, and, applied to
$S_l$ alone, $s_{S_l}(x)=s(x)-q_{S_l}(x)$. Define
$q_l^{\mathrm{true}}(x):=s(x)-s_{S_l}^{(l)}(x)$, so that by the definition of
$\Delta_l^{\mathrm{eff}}$ above,
$q_l^{\mathrm{true}}(x) = s(x)-s_{S_l}(x)-\Delta_l^{\mathrm{eff}}(x)
= q_{S_l}(x)-\Delta_l^{\mathrm{eff}}(x)$.
Define $s_S^\dagger(x):=s(x)-\sum_{l\in L(S)}q_l^{\mathrm{true}}(x)$: the
selector obtained by editing every touched block \emph{in isolation} and
combining the results additively, exactly as the idealized model combines
independent carriers. Substituting,
\[
s_S^\dagger(x) \;=\; s(x)-\sum_l\bigl(q_{S_l}(x)-\Delta_l^{\mathrm{eff}}(x)\bigr)
\;=\; s_S(x) + \sum_{l\in L(S)}\Delta_l^{\mathrm{eff}}(x).
\]
Setting $R_\times(x):=s_S^{\mathrm{true}}(x)-s_S^\dagger(x)$ (a
definition, not a further claim) gives
$s_S^{\mathrm{true}}(x)-s_S(x) = \sum_l\Delta_l^{\mathrm{eff}}(x) + R_\times(x)$,
which is \eqref{eq:multilayer}; every step so far is regrouping a finite sum
and substituting definitions, using no hypothesis on where block $l$ sits.
For the first pinning-down fact: if $S_l$ ablates only block $l$'s MLP, its
head's write-in is unedited, so the residual entering block $l$'s
normalization is the same in $F_{S_l}^{(l)}$ as in the clean network; the
only change to $x_{l+1}$ is the MLP's own projector acting on that same
input, which is exactly what the idealized $F_{S_l}$ also applies to the same
clean input, so $x_{l+1}$ agrees between the two, and hence so does every
downstream computation and the final selector: $\Delta_l^{\mathrm{eff}}(x)=0$.
For the second: if block $l$ feeds the readout directly, $F_{S_l}^{(l)}$ and
$F_{S_l}$ are exactly the objects Fact~\ref{prop:readout} compares,
which gives $s_{S_l}^{(l)}(x)=s_{S_l}(x)+\psi(\Delta_l(x))$, i.e.\
$\Delta_l^{\mathrm{eff}}(x)=\psi(\Delta_l(x))$. When $|L(S)|=1$, say
$L(S)=\{l\}$, editing block $l$ in isolation \emph{is} editing all of $S$, so
$s_{S_l}^{(l)}(x)=s_S^{\mathrm{true}}(x)=s_S^\dagger(x)$ and $R_\times(x)=0$
identically.
\end{proof}

\begin{remark}
\label{rem:multilayer-open}
Outside the two pinned-down cases, $\Delta_l^{\mathrm{eff}}(x)$ is exactly
the quantity the companion paper's own discussion of cross-layer propagation
already flagged as open: the effect, on the final selector, of a same-block
interaction at an \emph{earlier} layer once it has propagated through every
subsequent block. This proposition does not close that gap: it isolates the
quantity precisely enough to be measured on its own, separately from
$R_\times$, which is the contribution we take from it.
Section~\ref{sec:experiments} measures both $\sum_l\Delta_l^{\mathrm{eff}}(x)$
and $R_\times(x)$ on configurations where some touched layer is not the
last, where neither pinned-down case applies to every term of the sum.
\end{remark}

\begin{remark}
\label{rem:multilayer-scope}
$R_\times(x)$ is not a same-block phenomenon under a different name: it
vanishes whenever $S$ touches a single layer, and is generically present
whenever it touches two or more, regardless of whether any touched layer's
$S_l$ pairs a head with its own MLP: two heads at different layers with no
MLP anywhere in $S$ already generate a nonzero $R_\times$ in general, since
ablating the earlier one changes the residual the later one's own removed
mass is computed from, which $F_{S_l}^{(l)}$, by construction, evaluates on
the untouched residual instead. Because the architecture is a strictly
sequential composition of blocks (block $l$'s weights play no role in
producing $x_1,\dots,x_l$, only in producing $x_{l+1}$ onward), $R_\times(x)$
is entirely a \emph{forward} effect: an earlier ablated layer can change what
a later one's carrier reads, but not the reverse, so no separate treatment of
two touched layers $l<l'$ versus $l>l'$ is needed beyond tracking, for each
touched pair, which one comes first. We do not attempt a closed-form bound on
$R_\times(x)$ in this section: doing so requires a curvature bound for the
composition of an arbitrary number of further blocks, each contributing its
own attention nonlinearity in addition to the normalization--MLP composition
Section~\ref{sec:curvature} already bounds, and Section~\ref{sec:crossterm}
below is exactly the attempt to close that gap, reporting honestly how far it
reaches and where it stops. Section~\ref{sec:experiments} measures
$\sum_l\Delta_l^{\mathrm{eff}}(x)$ and $R_\times(x)$ separately on the trained
instances, using \eqref{eq:multilayer} to attribute the measured interaction
magnitude of a multi-layer configuration between the two.
\end{remark}

\section{An Exact Cross-Layer Interaction Identity}
\label{sec:crossterm}

Remark~\ref{rem:multilayer-scope} states that $R_\times(x)$ is generically
nonzero once $S$ touches two or more layers and declines to bound it. This
section isolates $R_\times(x)$ exactly, for two touched layers, as the same
kind of object $\Delta(x)$ already is in Fact~\ref{thm:interaction} (one
level up, across two entire blocks instead of two carriers inside one), and
names precisely what stands between this identity and a closed-form bound in
the style of Section~\ref{sec:curvature} below.

Fix $S=S_{l_1}\cup S_{l_2}$ touching exactly two blocks, $l_1<l_2$
(Remark~\ref{rem:multilayer-scope}: only the order matters). Write
$\mu_l(z):=\mathrm{Block}_l(z)-\mathrm{Block}_l^{S_l}(z)$ for the write-in
block $l$'s own weight edit deletes at input $z$, where
$\mathrm{Block}_l^{S_l}$ is block $l$'s edited map, exactly the
block-level analogue of $q_l^{\mathrm{true}}(x):=s(x)-s_{S_l}^{(l)}(x)$ from
the proof of Proposition~\ref{prop:multilayer}, one layer earlier in the
computation. Expanding via the identity used in the proof of
Fact~\ref{prop:readout} (there applied to the whole network; here to
block $l$ alone) gives the closed form
\begin{equation}
\label{eq:mudelta}
\mu_l(z) \;=\; \eta_l(z) + Q_lg_l\bigl(r_{1,l}(z)\bigr) - \Delta_l(z),
\end{equation}
with $\eta_l,Q_l,g_l,r_{1,l}$ block $l$'s own instances of
Section~\ref{sec:setup}'s notation and $\Delta_l$ block $l$'s own
instance of Fact~\ref{thm:interaction}'s interaction term: the within-block
result is not bypassed here, it is exactly what supplies $\mu_l$ when block
$l$'s own two carriers are both touched ($\mu_l=\eta_l$ when only the head is,
since then $Q_l=0$ and $\Delta_l=g_l(r_{1,l}(z)-\eta_l(z))-g_l(r_{1,l}(z))$
cancels the corresponding piece of $\eta_l$ exactly; $\mu_l=Q_lg_l(r_{1,l}(z))$
when only the MLP is, since then $\eta_l=0$ forces $\Delta_l\equiv0$ by
Fact~\ref{thm:interaction}).

For $(t_1,t_2)\in[0,1]^2$, let $H(t_1,t_2)$ be the selector of the network in
which block $l_1$'s map is $z\mapsto\mathrm{Block}_{l_1}(z)-t_1\mu_{l_1}(z)$,
block $l_2$'s is $z\mapsto\mathrm{Block}_{l_2}(z)-t_2\mu_{l_2}(z)$, and every
other block is unedited. $H$ is linear in $t_1$ and in $t_2$ separately at
each fixed underlying input, and is a composition of $C^\infty$ maps
(Fact~\ref{lem:jacobian} for every normalization layer; $\mathrm{SiLU}$ and
softmax are $C^\infty$), hence $H\in C^2([0,1]^2)$.

\begin{proposition}[Exact two-layer cross term]
\label{prop:crossterm}
With $H$ as above, $H(0,0)=s(x)$, $H(1,0)=s_{S_{l_1}}^{(l_1)}(x)$,
$H(0,1)=s_{S_{l_2}}^{(l_2)}(x)$, $H(1,1)=s_S^{\mathrm{true}}(x)$, and
\[
R_\times(x) \;=\; H(1,1)-H(1,0)-H(0,1)+H(0,0)
\;=\; \int_0^1\!\!\int_0^1 \frac{\partial^2 H}{\partial t_1\partial t_2}(t_1,t_2)\,dt_1\,dt_2 .
\]
\end{proposition}

\begin{proof}
The four corner values follow from the definitions of $\mathrm{Block}_l^{S_l}$
and $\mu_l$: $(t_1,t_2)=(0,0)$ edits nothing, giving $s(x)$; $(1,0)$ edits only
block $l_1$'s map, which is exactly $F_{S_{l_1}}^{(l_1)}$ by definition of that
object; $(0,1)$ is symmetric; $(1,1)$ edits both blocks' maps simultaneously
and every other block is untouched, which is $F_S^{\mathrm{true}}$. The first
equality is then the definition of $R_\times(x)$ used in the proof of
Proposition~\ref{prop:multilayer}, specialized to $|L(S)|=2$ (there,
$s_S^\dagger(x)=s(x)-\sum_l q_l^{\mathrm{true}}(x)$ with two terms, i.e.\
$-H(1,0)-H(0,1)+2H(0,0)$; adding $H(1,1)-H(0,0)$ and simplifying gives the
first equality). For the second: fixing $t_1$, the one-dimensional
fundamental theorem of calculus gives
$H(t_1,1)-H(t_1,0)=\int_0^1\partial_{t_2}H(t_1,t_2)\,dt_2$; since
$H\in C^2([0,1]^2)$, $t_1\mapsto\partial_{t_2}H(t_1,t_2)$ is $C^1$, and
applying the fundamental theorem of calculus again, in $t_1$, to
$t_1\mapsto\bigl[H(t_1,1)-H(t_1,0)\bigr]$ gives
\[
\bigl[H(1,1)-H(1,0)\bigr]-\bigl[H(0,1)-H(0,0)\bigr]
= \int_0^1\!\!\int_0^1\partial_{t_1}\partial_{t_2}H(t_1,t_2)\,dt_1\,dt_2,
\]
and the left side rearranges to $H(1,1)-H(1,0)-H(0,1)+H(0,0)$. This is exactly
the argument used in the proof of Fact~\ref{thm:interaction}, applied
twice, one dimension up.
\end{proof}

Attention's own input is likewise normalization-bounded here ($\mathrm{ao}$
reads $N(x_l)$, not $x_l$ directly), which raises the question of whether a
bound on attention's own Jacobian, in the pointwise style of
Section~\ref{sec:curvature} below (evaluated at the real weights and the real
probed input, not a supremum over all possible inputs), is available at all.
It is, for a single key/value token perturbed at a time, and we verify it
numerically before using it for anything.

\begin{proposition}[Local attention Jacobian bound]
\label{prop:attnjacobian}
Fix a head with query position $\tau$ and write $u_i:=N(x_i)$ for the
normalized input every head reads. Let $q:=W_Qu_\tau$,
$k_i:=W_Ku_i$, $v_i:=W_Vu_i$ for $i\le\tau$ (causal masking),
$p:=\mathrm{softmax}\bigl(q^\top k_\bullet/\sqrt{d_{\mathrm h}}\bigr)$, and
$a:=\sum_{i\le\tau}p_iv_i$ the head's output at $\tau$. For $j<\tau$,
\begin{equation}
\label{eq:attnjac}
\frac{\partial a}{\partial u_j} \;=\; p_j\,W_V \;+\; p_j\,(v_j-a)\otimes\frac{W_K^\top q}{\sqrt{d_{\mathrm h}}},
\end{equation}
exactly, and consequently
\begin{equation}
\label{eq:attnjacbound}
\left\|\frac{\partial a}{\partial u_j}\right\|_{\mathrm{op}}
\;\le\; p_j\|W_V\|_{\mathrm{op}} + p_j\|v_j-a\|\,\frac{\|W_K\|_{\mathrm{op}}\|q\|}{\sqrt{d_{\mathrm h}}}
\;\le\; p_j\|W_V\|_{\mathrm{op}}\left(1+\frac{2d\,\|\gamma\|_\infty^2\|W_Q\|_{\mathrm{op}}\|W_K\|_{\mathrm{op}}}{\sqrt{d_{\mathrm h}}}\right),
\end{equation}
the second inequality using $\|u_i\|\le\sqrt d\,\|\gamma\|_\infty$ (the same
bound Section~\ref{sec:curvature}'s Proposition~\ref{prop:curvature} uses)
to close $\|v_j\|,\|q\|$ and $\|v_j-a\|\le2\max_i\|v_i\|$.
\end{proposition}

\begin{proof}
$u_j$ affects $a$ only through $v_j=W_Vu_j$ directly and through $k_j=W_Ku_j$,
which affects $p$ (every logit $z_i:=q^\top k_i/\sqrt{d_{\mathrm h}}$ for
$i\ne j$ is independent of $u_j$). The direct term contributes
$p_j\,\partial v_j/\partial u_j=p_jW_V$. For the softmax term, the standard
softmax Jacobian $\partial p_i/\partial z_j=p_i(\delta_{ij}-p_j)$ and
$\partial z_j/\partial u_j=(W_K^\top q)^\top/\sqrt{d_{\mathrm h}}$ give
$\sum_iv_i\,\partial p_i/\partial u_j
=\bigl[p_jv_j-p_j\sum_ip_iv_i\bigr](W_K^\top q)^\top/\sqrt{d_{\mathrm h}}
=p_j(v_j-a)\otimes(W_K^\top q)/\sqrt{d_{\mathrm h}}$, using
$\sum_ip_iv_i=a$. Summing the two terms gives \eqref{eq:attnjac}. The first
inequality in \eqref{eq:attnjacbound} is the triangle inequality and
submultiplicativity of the operator norm; the second substitutes
$\|q\|=\|W_Qu_\tau\|\le\|W_Q\|_{\mathrm{op}}\sqrt d\|\gamma\|_\infty$ and
$\|v_j-a\|\le\|v_j\|+\|a\|\le2\|W_V\|_{\mathrm{op}}\sqrt d\|\gamma\|_\infty$
(the last step using $\|a\|\le\max_i\|v_i\|$, a convex combination of the
$v_i$, and $\|v_j\|\le\|W_V\|_{\mathrm{op}}\sqrt d\|\gamma\|_\infty$
identically).
\end{proof}

\begin{remark}[Numerical check]
\label{rem:attnjacobian-check}
Both bounds of \eqref{eq:attnjacbound} were checked, in the same
refutation-test spirit as Remark~\ref{rem:curvature-check} below, against
finite-difference estimates of $\|\partial a/\partial u_j\|_{\mathrm{op}}$ on
Qwen2.5-1.5B-Instruct's real weights and a real forward pass: $6$
(layer, head) pairs spanning layers $1$ to $28$, $2$ key positions each, $24$
random unit perturbations per site: $0$ violations of either bound in
$12$ sites. The first (tight) inequality of \eqref{eq:attnjacbound}, using
the real $p_j,v_j,a,q$ at the probed point, exceeds the empirical estimate
by a factor of $12$ to $111$ (median $71$), tighter than
Proposition~\ref{prop:curvature}'s own looseness ($6.5\times10^2$ to
$2.4\times10^3$) on the same kind of check. The second (weight-only,
normalization-closed) inequality is looser by four to six orders of
magnitude (median $1.7\times10^5$, max $4.2\times10^6$): finite and never
violated, but not a numerically useful estimate as stated. Script:
\texttt{notebook/\allowbreak verify\_attention\_jacobian\_bound.jl}.
\end{remark}

\begin{remark}[What this closes, and what it does not]
\label{rem:crossterm-gap}
Proposition~\ref{prop:crossterm} turns ``$R_\times(x)$ is generically
nonzero, left open'' into an exact double integral of a mixed second
derivative, and it isolates precisely one missing ingredient rather than
leaving the whole quantity unaddressed. $\partial_{t_1}\partial_{t_2}H$ is
nonzero only through how $t_1$ changes the residual reaching block $l_2$'s
input: $\mu_{l_1}$ is evaluated at a fixed point, since block $l_1$ has no
edited block upstream of itself, but $\mu_{l_2}$ is evaluated at
$x_{l_2}(t_1)$, a $t_1$-dependent residual once $t_1>0$. Bounding
$\partial_{t_1}\partial_{t_2}H$ therefore reduces, to leading order, to
bounding the operator norm of the Jacobian (with respect to its own
input) of the composition of every block strictly between $l_1$ and
$l_2$. For each such block, the normalization--MLP half of that Jacobian is
already bounded inside the proof of Proposition~\ref{prop:curvature} below
(there: $\|DM(z)\|_{\mathrm{op}}\le(1+s_1)\|W_1\|_{\mathrm{op}}
\|W_2\|_{\mathrm{op}}\|W_3\|_{\mathrm{op}}\|z\|$ and
$\|DN(r)\|_{\mathrm{op}}\le\|\gamma\|_\infty\rho(r)$), and
Proposition~\ref{prop:attnjacobian} now supplies the attention half, for a
single perturbed key/value token, at the pointwise standard
Proposition~\ref{prop:curvature} itself uses, verified rather than merely
plausible, per Remark~\ref{rem:attnjacobian-check}. What remains open is
composing this one-token, one-layer statement into a bound on the full
Jacobian of a composition of an arbitrary number of further blocks: summing
\eqref{eq:attnjacbound} over every key/value position gives a
per-layer bound (using $\sum_jp_j=1$, this collapses to
$\|W_V\|_{\mathrm{op}}(1+2d\|\gamma\|_\infty^2\|W_Q\|_{\mathrm{op}}
\|W_K\|_{\mathrm{op}}/\sqrt{d_{\mathrm h}})$, independent of the individual
$p_j$), but chaining such per-layer factors across every block strictly
between $l_1$ and $l_2$ multiplies them together, and whether that product
stays controlled (rather than compounding the four-to-six-order-of-magnitude
looseness already present in one layer's closed-form bound into something
vacuous after a handful of layers) is exactly the empirical question already
raised, in the companion paper's own discussion of cross-layer propagation,
and we do not resolve it here. Unlike the single-layer case, we have not run
the corresponding multi-layer refutation test, so we report the per-layer
bound as verified and the chained, multi-layer one as unverified and likely
loose, not as closed. The same ingredient bears on Gap 2 (the companion
paper's own open question for an intermediate-layer head): $H$ above
degenerates to a single-parameter path when $|L(S)|=1$, and
$\Delta_l^{\mathrm{eff}}$ is $\psi$ applied to exactly the propagated
analogue of $R(x)$ in Fact~\ref{thm:interaction}: the same per-layer
attention bound applies at each step of that propagation, with the same
caveat about chaining it across depth.
\end{remark}

\begin{remark}[No identifiable sign law]
\label{rem:signlaw}
Differentiating $H$ twice at $(0,0)$ shows $\partial_{t_1}\partial_{t_2}H(0,0)$
is, to leading order, a \emph{bilinear} pairing between how much of block
$l_1$'s removed mass reaches block $l_2$'s input after propagating through
every intervening block, and how sensitively block $l_2$'s own removed-mass
map $\mu_{l_2}$ depends on its input there, not the product of two scalars
this paper has already named. Neither $\Delta_{S_{l_1}}\cdot\Delta_{S_{l_2}}$
nor $\Delta_{S_{l_1}}+\Delta_{S_{l_2}}$ is, in general, the sign of a bilinear
pairing between a propagated vector and a Jacobian, so we would not expect
either to predict $\mathrm{sign}\,R_\times(x)$ as a rule, a structural
reason, not merely an empirical failure to find one. Section~\ref{sec:experiments}
gives no support for either candidate: writing $\Delta_H:=s(x)-s_H(x)$ and
$\Delta_M:=s(x)-s_M(x)$ for the two single-site zero-ablations measured
there, the identity $-R_\times(x)=\Delta_{HM}-\Delta_H-\Delta_M$ (immediate
from the definitions, with $S_{l_1}=\{H\}$, $S_{l_2}=\{M\}$) makes
Section~\ref{sec:experiments}'s already-reported ``interaction'' column exactly
$-R_\times(x)$ on real Qwen weights. Its sign matches
$\mathrm{sign}(\Delta_H+\Delta_M)$ on two of the three applicable instances
(Mary/John, Alice/Bob) and not the third: on Sarah/Tom, $\Delta_H,\Delta_M>0$
so that candidate rule predicts $-R_\times>0$, but the measured
$-R_\times=-0.117$. Three points settle nothing statistically, and we report
the check only because it was already available at zero additional cost, not
as evidence for or against any rule. We were unable to identify a condition,
stated in terms of quantities this paper already defines, that predicts
$\mathrm{sign}\,R_\times(x)$ in general; the bilinear structure above is a
reason to expect that no simple one exists, which we record as a negative
finding rather than an unexplored question.
\end{remark}

\section{A Closed-Form Curvature Constant}
\label{sec:curvature}

Fact~\ref{thm:interaction} assumes a bound $\|D^2g\|_{\mathrm{op}}\le\Lambda$
on the segment in order to make the remainder second order, but leaves
$\Lambda$ unexhibited. We now compute it, in closed form, from the trained
weights. The only ingredient still missing is the second derivative of the
normalization; $DN$ already involves $\rho(r)^3$, so $D^2N$ is a third-order
tensor, but it collapses to four terms.

\begin{lemma}[Second derivative of the normalization]
\label{lem:d2norm}
For every $r\in\mathbb R^d$ and all $u,v\in\mathbb R^d$,
\[
D^2N(r)[u,v] \;=\; -\frac{\rho(r)^3}{d}\,\gamma\odot
\bigl(\langle r,v\rangle\,u+\langle r,u\rangle\,v+\langle u,v\rangle\,r\bigr)
\;+\;\frac{3\rho(r)^5}{d^2}\,\langle r,u\rangle\langle r,v\rangle\,\gamma\odot r ,
\]
and consequently
$\|D^2N(r)\|_{\mathrm{op}}\le 6\,\|\gamma\|_\infty\,\rho(r)^2/\sqrt d$.
\end{lemma}

\begin{proof}
Differentiating $\partial N_i/\partial r_j=\gamma_i\rho\bigl(\delta_{ij}
-\tfrac{\rho^2}{d}r_ir_j\bigr)$ from Fact~\ref{lem:jacobian} once more, and
using $\partial\rho/\partial r_k=-\tfrac1d\rho^3r_k$ throughout, gives
\[
\frac{\partial^2N_i}{\partial r_j\partial r_k}
=\gamma_i\Bigl[-\frac{\rho^3}{d}\bigl(\delta_{ij}r_k+\delta_{ik}r_j+\delta_{jk}r_i\bigr)
+\frac{3\rho^5}{d^2}r_ir_jr_k\Bigr],
\]
and contracting against $u_jv_k$ yields the displayed identity. For the bound,
take $\|u\|=\|v\|=1$ and use $\|\gamma\odot w\|\le\|\gamma\|_\infty\|w\|$: the
first group is at most $3\|\gamma\|_\infty\rho^3\|r\|/d$ and the second at most
$3\|\gamma\|_\infty\rho^5\|r\|^3/d^2$. Both are at most
$3\|\gamma\|_\infty\rho^2/\sqrt d$, since
$\rho(r)\|r\|=\|r\|\bigl(\tfrac1d\|r\|^2+\varepsilon\bigr)^{-1/2}\le\sqrt d$.
\end{proof}

\begin{proposition}[Closed-form curvature bound]
\label{prop:curvature}
Let $M(z)=W_2\,h(z)$ with $h(z):=\mathrm{SiLU}(W_1z)\odot W_3z$, let
$g=M\circ N$, and put $s_1:=\sup_t|\mathrm{SiLU}'(t)|$ and
$s_2:=\sup_t|\mathrm{SiLU}''(t)|$. Both are finite universal constants:
$s_2=\tfrac12$, attained at $0$, and $s_1=1.0998\ldots$. Then for every
$r\in\mathbb R^d$,
\begin{equation}
\label{eq:curvature}
\|D^2g(r)\|_{\mathrm{op}} \;\le\;
\|\gamma\|_\infty^2\,\rho(r)^2\,
\|W_1\|_{\mathrm{op}}\|W_2\|_{\mathrm{op}}\|W_3\|_{\mathrm{op}}
\Bigl(6+8s_1+s_2\sqrt d\,\|\gamma\|_\infty\|W_1\|_{\mathrm{op}}\Bigr).
\end{equation}
Consequently Fact~\ref{thm:interaction}'s curvature hypothesis holds on the
segment joining $r_1(x)-\eta(x)$ and $r_1(x)$ with $\Lambda$ given by the
right-hand side of \eqref{eq:curvature} evaluated at
$\rho_{\max}(x):=\sup_{t\in[0,1]}\rho\bigl(r_1(x)-t\eta(x)\bigr)$ (a
quantity computable from the trained weights and the probed input alone), and,
since $\rho\le\varepsilon^{-1/2}$ everywhere, with a constant uniform in $x$.
\end{proposition}

\begin{proof}
That $s_2=\tfrac12$: from $\mathrm{SiLU}(t)-\mathrm{SiLU}(-t)=t\sigma(t)+t\sigma(-t)=t$
we get $\mathrm{SiLU}''(t)=\mathrm{SiLU}''(-t)$, and
$\mathrm{SiLU}''(t)=\sigma'(t)\bigl(2+t(1-2\sigma(t))\bigr)$ is maximized at
$t=0$, where it equals $2\sigma'(0)=\tfrac12$.

Since $M$ is linear in $h$, $D^2M(z)[u,v]=W_2\,D^2h(z)[u,v]$ exactly, and
writing $a:=W_1z$,
\[
D^2h(z)[u,v]=\bigl(\mathrm{SiLU}''(a)\odot W_1u\odot W_1v\bigr)\odot W_3z
+\bigl(\mathrm{SiLU}'(a)\odot W_1u\bigr)\odot W_3v
+\bigl(\mathrm{SiLU}'(a)\odot W_1v\bigr)\odot W_3u .
\]
Using $\|p\odot q\|\le\|p\|_\infty\|q\|$ and $\|p\|_\infty\le\|p\|$ termwise,
\[
\|D^2M(z)\|_{\mathrm{op}}\;\le\;\|W_2\|_{\mathrm{op}}
\bigl(s_2\|W_1\|_{\mathrm{op}}^2\|W_3\|_{\mathrm{op}}\|z\|
+2s_1\|W_1\|_{\mathrm{op}}\|W_3\|_{\mathrm{op}}\bigr).
\]
Similarly
$Dh(z)[u]=\bigl(\mathrm{SiLU}'(a)\odot W_1u\bigr)\odot W_3z+\mathrm{SiLU}(a)\odot W_3u$
and $|\mathrm{SiLU}(t)|\le|t|$ give
$\|DM(z)\|_{\mathrm{op}}\le(1+s_1)\|W_1\|_{\mathrm{op}}\|W_2\|_{\mathrm{op}}\|W_3\|_{\mathrm{op}}\|z\|$.

By the chain rule
$D^2g(r)[u,v]=D^2M(N(r))\bigl[DN(r)u,DN(r)v\bigr]+DM(N(r))\,D^2N(r)[u,v]$, so
\[
\|D^2g(r)\|_{\mathrm{op}}\le\|D^2M(N(r))\|_{\mathrm{op}}\|DN(r)\|_{\mathrm{op}}^2
+\|DM(N(r))\|_{\mathrm{op}}\|D^2N(r)\|_{\mathrm{op}} .
\]
Substituting $\|DN(r)\|_{\mathrm{op}}\le\|\gamma\|_\infty\rho(r)$
(Fact~\ref{lem:jacobian}), $\|D^2N(r)\|_{\mathrm{op}}\le6\|\gamma\|_\infty\rho(r)^2/\sqrt d$
(Lemma~\ref{lem:d2norm}), and the uniform activation bound
$\|N(r)\|\le\sqrt d\,\|\gamma\|_\infty$ (which follows from
$\rho(r)\|r\|\le\sqrt d$), the first term is at most
$\|\gamma\|_\infty^2\rho^2\|W_1\|\|W_2\|\|W_3\|\bigl(s_2\sqrt d\|\gamma\|_\infty\|W_1\|+2s_1\bigr)$
and the second at most
$6(1+s_1)\|\gamma\|_\infty^2\rho^2\|W_1\|\|W_2\|\|W_3\|$. Adding gives
\eqref{eq:curvature}.
\end{proof}

Proposition~\ref{prop:curvature} has the shape one would hope for: it
introduces no architectural quantity beyond the three MLP spectral norms and
the $\gamma$ that already govern the leading-order term through
$\|DN\|_{\mathrm{op}}\le\|\gamma\|_\infty\rho(r)$, and it is multiplied by
$\rho(r)^2$ rather than growing with $\|r_1(x)\|$: the MLP's actual input is
normalization-bounded, so the remainder cannot blow up merely because the
residual stream does.

\begin{remark}[Numerical check and looseness]
\label{rem:curvature-check}
Both bounds were checked against fourth-point finite-difference estimates of
$D^2N$ and $D^2g$ on the five trained instances of the companion paper's own
marker-task experiments, at real residual-stream points and over random
direction pairs: $0$ violations in $1200$ samples for each. This is a
refutation test, not a confirmation: a single violation would indicate an
algebraic error above. It is worth being explicit about the price of the
product-of-norms argument: on these weights \eqref{eq:curvature} exceeds the
largest sampled curvature by a factor of $6.5\times10^2$ to $2.4\times10^3$,
while Lemma~\ref{lem:d2norm} is loose by only $12.9$--$15.7\times$. The
composition step, not the normalization, is where the slack accumulates.
$\Lambda$ is therefore best read as a certificate that the remainder is
second order with a computable constant, not as a sharp numerical estimate of
it. Script: \texttt{notebook/\allowbreak verify\_curvature\_bound.jl}.
\end{remark}

\section{Illustrative Experiments}
\label{sec:experiments}

This section reports two checks. The first reuses the companion paper's own
marker-task checkpoints (five small transformers trained from scratch on a
synthetic conditional task, detailed in Appendix~\ref{app:repro}) to measure
Proposition~\ref{prop:multilayer}'s same-block/cross-layer split on real
configurations spanning several layers. The second, and the paper's main
empirical result, checks the qualitative three-way pattern of collapse,
dissociation, and interaction on a real pretrained model's own emergent
circuit, found rather than designed.

\paragraph{Same-block versus cross-layer.} The companion paper's own
experiments fix, per trained instance, three canonical ablation
configurations (a single carrier, all carriers in the shallowest
carrier-bearing layer, and all carriers) and separately probe every other
non-empty subset of each instance's carrier set, for $39$ configurations in
total across the five instances; the full interaction-magnitude-versus-
idealization-fidelity relationship across all $39$ is plotted as a figure in
the companion paper, which we do not reproduce here since it is not this
paper's own result. Of those $39$, $18$ touch two or more layers.
Proposition~\ref{prop:multilayer} decomposes the interaction magnitude of any
subset spanning several layers into a sum of same-block terms
$\Delta_l^{\mathrm{eff}}$ (each identically zero if that layer's own carriers
are MLP-only, and each exactly Fact~\ref{prop:readout}'s already-bounded
$\psi(\Delta_l(x))$ only for the layer that feeds the readout directly, an
open, propagated quantity otherwise), plus a cross-layer remainder
$R_\times$ on which we made no analytical claim either. We measured both
sides of \eqref{eq:multilayer} on the $18$ multi-layer configurations,
reusing the same subspaces, and on $80$ fresh probe pairs per configuration
($1440$ pair-configurations in total). The identity itself reproduces to $0$
in floating point on every one of them, the same kind of algebraic-identity
check used elsewhere in the companion paper's own robustness analysis. What
it lets us attribute is less comfortable than a clean same-block/cross-layer
split would suggest. Pooled across all $1440$ measurements, the median
same-block sum ($0.31$ logits) and the median total interaction ($0.30$) are
close, but the median cross-layer remainder ($0.59$) is roughly twice
either, and the same-block sum alone exceeds the \emph{total} interaction on
$57\%$ of pair-configurations, the cross-layer remainder on $63\%$, so the
two routinely partial-cancel rather than one dominating. The pattern is not
uniform: configurations that ablate two heads at different layers with no
MLP anywhere in $S$ show the largest same-block sums relative to the total
(e.g.\ a layer-$1$ and a layer-$2$ head on instance 2, median same-block
$3.85$ against a median total of only $0.11$; neither layer's own carriers
are MLP-only, so neither term is pinned to zero, and here the two, whatever
their individual size, largely cancel against $R_\times$), whereas
configurations that pair a head with its own layer's MLP at one of the
touched layers show same-block and cross-layer terms closer in size to the
total itself (e.g.\ seed 22's layer-$1$ head+MLP together with a layer-$2$
head: median same-block $0.21$, median total $0.37$). The reading we take
from this is not that Fact~\ref{thm:interaction} is wrong about the
same-block mechanism (the identity that isolates it is exact, and every
same-block term we measured nonzero is consistent with that mechanism being
active at the layer it sits at) but that for a subset spanning more than one
layer, the same-block sum alone is not a reliable proxy for the net
interaction: $R_\times$ is routinely the same order of magnitude or larger,
and a bound on same-block interaction alone, without one on $R_\times$,
would not let a reader anticipate whether the two reinforce or cancel. That
bound is the open question Remark~\ref{rem:multilayer-scope} states and does
not close. Script: \texttt{notebook/\allowbreak verify\_crosslayer\_decomposition.jl}.

\subsection{A real pretrained model: an emergent circuit in
Qwen2.5-1.5B-Instruct}
\label{sec:qwen}

The companion paper's own experiments train networks from scratch
specifically to exhibit the two-branch structure of
Definition~\ref{def:pair}, and the same-block-versus-cross-layer measurement
above reuses those same checkpoints. This final check instead asks whether
that structure, and the carriers realizing it, can be found \emph{as-is}
inside a real instruction-tuned model that was never trained with either
paper in mind: Qwen2.5-1.5B-Instruct (28 layers, 12 query heads, 2 key/value
heads, GQA, $d=1536$), loaded natively (not through a PyTorch bridge), with
logits matching the reference implementation to within $3.7\times10^{-5}$
absolute on the prompts used here.

\paragraph{Finding a matched pair in real language.} Definition~\ref{def:pair}
needs two inputs sharing a fixed selector $s(x)=\mathrm{logit}(c_1)-\mathrm{logit}(c_2)$
with $s(x_A)>0>s(x_B)$ and a healthy margin on both, not merely the right
sign. Four candidate mechanisms were fixed in advance, before any model was
probed, each with $5$ lexical instances: indirect object identification (IOI;
swapping which of two named entities gave an item to the other,
\cite{wang2023}), a sentiment cue (loved/hated $\to$ positive/negative
adjective), temporal order (who left first), and a magnitude comparator
(larger/smaller number) kept as a designed-rather-than-found fallback. Only
IOI was clean on all $5$ instances at margin $>1.0$ logit (e.g.\ Mary/John:
$g_A=8.20$, $g_B=-6.37$); sentiment was clean on $3/5$, temporal order on
$2/5$, and the magnitude comparator on $0/5$ (rejected, as flagged in advance).
IOI is retained for everything below; the other three families are reported
here only because they were pre-registered, not because they matter further.

\paragraph{Method.} On the $5$ IOI instances, we run a greedy
activation-patching search (Wang et al.'s original method for this task,
\cite{wang2023}; the same causal-tracing operation this paper uses to reason
about carriers throughout) over $364$ candidate sites (the output of every
attention head ($28\times12$) and every MLP block ($28$)), using this
paper's own \texttt{greedy\_patch\_search!}\allowbreak/\texttt{backward\_prune!}
implementation with the recovery metric restricted to $s(x)$ at the final
token rather than the full vocabulary, then prunes the found set for
redundant sites. Search depth was capped at $6$ sites per instance, fixed
before any instance was run.

\paragraph{A shared circuit.} Figure~\ref{fig:qwen-freq} shows how often each
site appears across the $5$ pruned circuits. One head, layer $25$/head $9$,
appears in all $5$; one more, layer $19$/head $6$, in $4$ of $5$; the
remaining $16$ sites appear in at most $3$ circuits each, mostly in exactly
one: a small shared core plus instance-specific idiosyncratic support,
the same redundant-code pattern reported elsewhere in superposed and
self-repairing networks \cite{elhage2022,mcgrath2023} and already discussed
in Section~\ref{sec:related} above. On every one of the $5$ instances,
\texttt{backward\_prune!} removed nothing: all $6$ sites found by the greedy
search were still individually necessary at its tolerance, so
Figure~\ref{fig:qwen-freq} reports circuits with no known-superfluous site
left in them. Figure~\ref{fig:qwen-traj} shows the recovery trajectory
itself: $4$ of $5$ instances reach recovery $\ge0.96$ within the $6$-site
budget; Lucy/Sam plateaus at $0.673$ and is the one instance where the
search did not find enough of the mechanism within budget.

\begin{figure}[htbp]
\centering
\includegraphics[width=.92\linewidth]{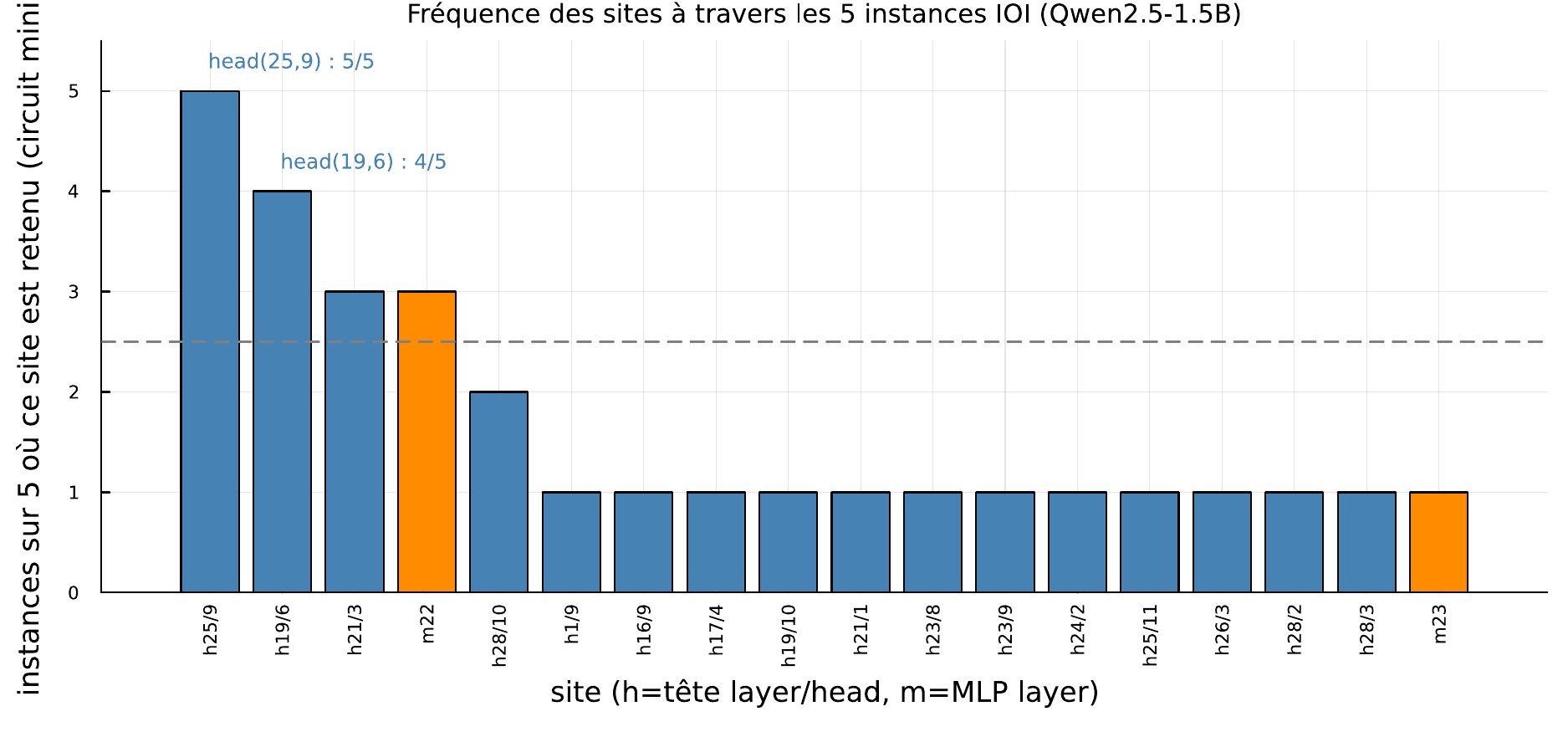}
\caption{How many of the $5$ pruned IOI circuits on Qwen2.5-1.5B-Instruct
contain each site (h=attention head layer/head, m=MLP layer). One head
(layer $25$/head $9$) is shared by all $5$; one more (layer $19$/head $6$) by
$4$; the other $16$ sites are idiosyncratic to $1$--$3$ instances.}
\label{fig:qwen-freq}
\end{figure}

\begin{figure}[htbp]
\centering
\includegraphics[width=.8\linewidth]{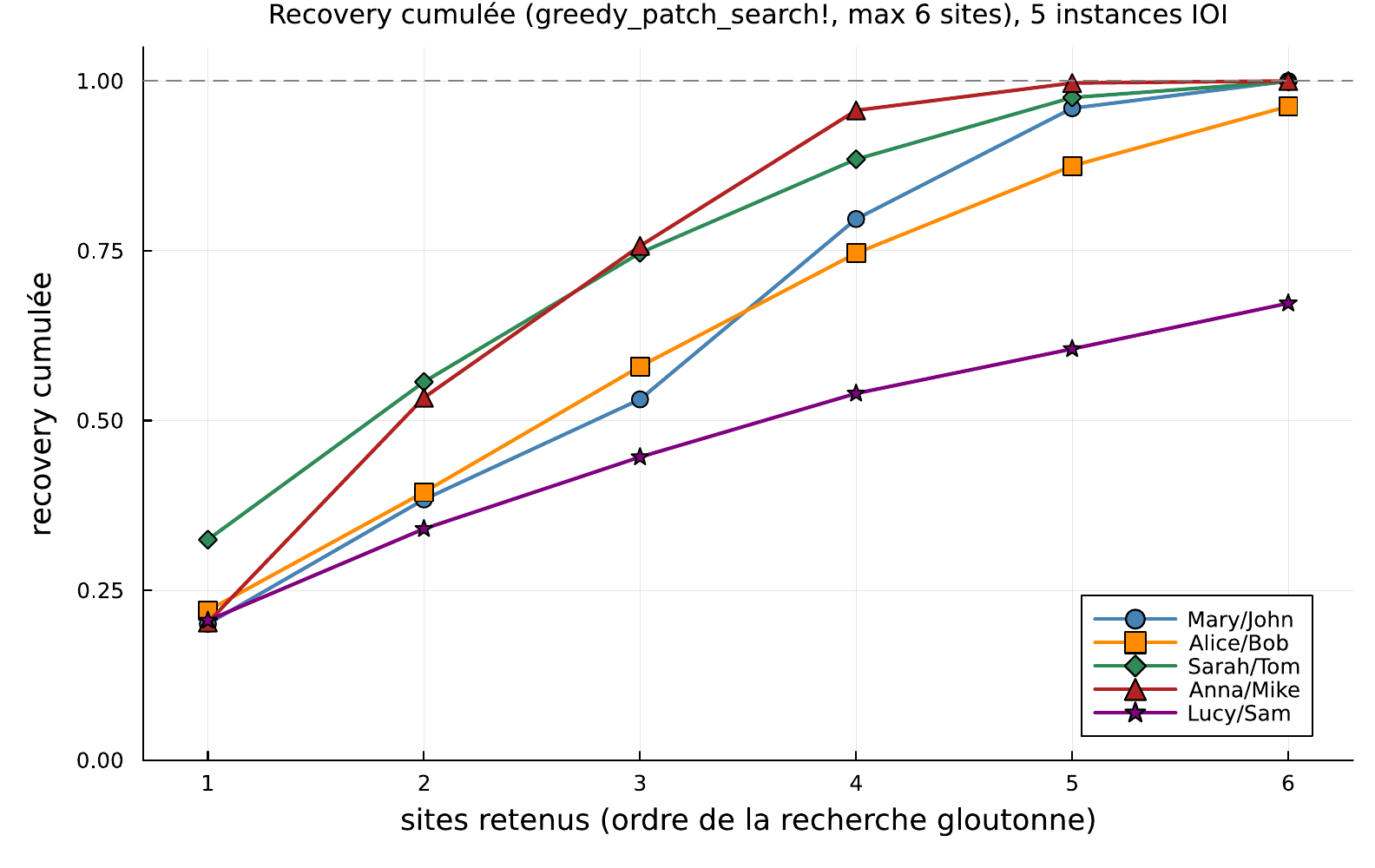}
\caption{Cumulative recovery of the greedy search, one curve per IOI
instance. Four instances converge to recovery $\ge0.96$ by the $6$-site
budget; Lucy/Sam plateaus at $0.673$, the weakest of the five.}
\label{fig:qwen-traj}
\end{figure}

\paragraph{Three probes, and what operation each one actually performs.}
On the pruned circuit of each instance we measure three quantities directly
analogous to this paper's central claims, and we state precisely which
operation realizes each one before reporting numbers, since the three are
not the same kind of intervention.

\emph{Collapse.} With every site in the pruned circuit forced to $0$
(not the frozen-activation prediction used elsewhere in this paper, but a
full recomputation of everything downstream, since \texttt{patch\_node!}
invalidates and re-derives every consumer), we compare the selector on
$x_A$ and on $x_B$ separately. Forcing a single component's activation to a
fixed value and letting the rest of the network recompute is, node for node,
what Fact~\ref{prop:patchedit} identifies with editing that
component's own weights by the corresponding projector; here the projector
is the identity on the component's entire output (dimension $d_{\mathrm
h}=128$ for a head, so still $\ll d=1536$ in the sense of
Assumption~\ref{ass:lowrank}; dimension $d=1536$, the whole residual, for an
MLP, so not low-rank in that sense at all, though no result in this paper
uses $\dim(\mathcal C)\ll d$ quantitatively). This is a coarser carrier than
the single donor--receiver direction estimated from data in the companion
paper's own experiments, but Remark~\ref{rem:bridge} and
Fact~\ref{prop:patchedit} are stated for an arbitrary projector, not
only a low-rank one, so the single-site case is a literal instance of the
deletion \eqref{eq:ablation}, not merely an activation-level analogue of it.
Jointly forcing several sites to $0$ is not separately covered by
Fact~\ref{prop:patchedit} (stated for one component), though the same
argument extends to it: a component whose output is pinned to a fixed value
regardless of its own input is, by construction, unaffected by what any
other pinned component does upstream of it, so the jointly-recomputed result
again coincides node for node with jointly editing every corresponding
weight matrix to zero. We did not verify this joint extension independently
of that argument, and report it on that basis.

\emph{Dissociation.} We compare this same joint zero-ablation, measured on
$x_A$'s own forward pass, against the ordinary activation-patching recovery
already used to find the circuit (corrupted run $x_B$ patched toward clean
$x_A$ at the same sites). Fact~\ref{thm:dissociation} is precisely the
claim that patching one input's activations into another and ablating a
component within a single input's own forward pass are, in general,
different operations; here they manifestly are different operations, one
substituting a value computed on a different token sequence, the other
forcing a component's own contribution to zero on the same sequence, and we
report the gap between their effect sizes as exactly that: a dissociation
between two operations, not an approximation error between two estimates of
the same one.

\emph{Interaction.} On the $3$ instances whose pruned circuit contains both
a head and an MLP, we take the first head and first MLP found (an arbitrary
tie-break among the pruned sites, not a choice aimed at any particular
layer pair) and measure the true, jointly-recomputed effect of ablating both
against the sum of their two individually-measured effects: exactly the
frozen-activation-versus-true-network gap this paper calls the
\emph{interaction magnitude} above, at the selector level. Two of the three
pairs found this way have the MLP carrier in an \emph{earlier} layer than
the head carrier (mlp layer $22$/$23$ before head layer $25$), so the
head's ablation cannot be reaching that MLP's own computation at all; the
third (Sarah/Tom: head layer $19$ before mlp layer $22$) has the right order
but the two are three transformer blocks apart, not separated by the single
RMSNorm of a shared block that Fact~\ref{thm:interaction} formalizes. None
of the three pairs instantiates that fact's same-block head-then-MLP
configuration; what nonzero interaction they show instead bears on this
paper's own open cross-multi-layer generalization of
Section~\ref{sec:multilayer} (Proposition~\ref{prop:multilayer},
Remark~\ref{rem:multilayer-open}), not on the companion paper's theorem
itself. More precisely, writing $\Delta_H:=s(x)-s_H(x)$ and
$\Delta_M:=s(x)-s_M(x)$ for the two single-site zero-ablations and
$\Delta_{HM}$ for their joint ablation, this ``interaction'' column is
exactly $\Delta_{HM}-\Delta_H-\Delta_M=-R_\times(x)$ of
Section~\ref{sec:crossterm} with $S_{l_1}=\{H\}$, $S_{l_2}=\{M\}$: these three
numbers are not merely analogous to $R_\times(x)$, they are real,
non-negligible measured values of it on a genuine $1.5$B-parameter pretrained
network, at the two specific (non-same-block) layer pairs each instance
happens to provide. Remark~\ref{rem:signlaw} above uses these same three
numbers in its search for a sign rule and finds none.

\begin{table}[htbp]
\centering
\small
\begin{tabular}{lccccc}
\toprule
Instance & Recovery & Collapse ratio & Dissociation gap & Interaction (rel.) \\
\midrule
Mary/John  & $1.000$ & $0.245$ & $0.678$ & $0.124$ \\
Alice/Bob  & $0.963$ & $0.224$ & $0.489$ & $0.327$ \\
Sarah/Tom  & $0.999$ & $0.168$ & $0.546$ & $0.066$ \\
Anna/Mike  & $1.000$ & $0.761$ & $0.852$ & n/a (0 MLP in circuit) \\
Lucy/Sam   & $0.673$ & $0.396$ & $0.224$ & n/a (0 MLP in circuit) \\
\bottomrule
\end{tabular}
\caption{The three probes on the $5$ IOI instances. Recovery is
\texttt{greedy\_patch\_search!}'s final joint activation-patching recovery
(cross-input, no weight-edit reading). Collapse ratio is
$|s_{\text{ablated}}(x_A)-s_{\text{ablated}}(x_B)|/|s(x_A)-s(x_B)|$ after
jointly zero-ablating the pruned circuit on each input separately ($0$ =
full collapse onto a shared value, $1$ = no collapse). Dissociation gap is
the absolute difference between the patching-recovery and the same-circuit
ablation effect on $x_A$ alone. Interaction (rel.) is $|\Delta_{HM}-\Delta_H-\Delta_M|/(|\Delta_H|+|\Delta_M|)$
for the first head/MLP pair in the pruned circuit, where applicable.}
\label{tab:qwen}
\end{table}

\paragraph{Honest verdict.} Collapse is clear-to-partial on $4$ of $5$
instances (ratio $0.168$--$0.396$) and absent on the fifth: Anna/Mike
retains a ratio of $0.761$, i.e.\ ablating its own $6$-site circuit barely
narrows the gap between the two branches, the same kind of instance-specific
exception the companion paper reports elsewhere on its own marker-task
seeds rather than rounds past. Dissociation is never small: the gap between
the patching effect and the ablation effect is at least $0.224$ and as large
as $0.852$ on all $5$ instances, i.e.\ these two interventions never
coincide on this circuit, consistent with Fact~\ref{thm:dissociation}'s
claim that they are different operations in general, though absent a
theorem-derived target value this is evidence of a real, always-present gap
rather than a quantitative confirmation. Interaction is measurable on only
$3$ of $5$ instances (the other two circuits contain no MLP at all, so the
probe does not apply, not that it was skipped), and among those three it
ranges from small ($0.066$, Sarah/Tom) to non-negligible ($0.327$,
Alice/Bob), and, as argued above, none of the three pairs sits in the
same-block configuration Fact~\ref{thm:interaction} was proven for, so this
is at most suggestive of this paper's own open multilayer case, not a test
of the companion theorem itself. We report all three probes as a genuinely
mixed result on a real model neither trained nor selected for either
paper's claims: a real, cross-instance shared carrier does emerge
(Figure~\ref{fig:qwen-freq}), collapse and dissociation both occur on most
instances and both fail to occur cleanly on at least one, and the
interaction probe (the one closest, in its literal operation, to this
paper's own weight-space machinery) lands mostly outside the exact scope any
theorem here was proven for.

\section{Discussion}
\label{sec:discussion}

\paragraph{What is established.} Proposition~\ref{prop:multilayer} gives an
exact decomposition of the multi-layer interaction into same-block terms
plus one remainder, for an arbitrary ablated subset spanning any number of
layers, with no approximation anywhere in the identity itself. Two of the
three quantities appearing in it are fully pinned down: a same-block term is
exactly zero when that layer's own carriers are MLP-only, and exactly the
companion paper's bounded interaction term when the layer feeds the readout
directly. Proposition~\ref{prop:crossterm} then isolates the third quantity,
the cross-layer remainder $R_\times$, exactly, for two touched layers, as a
double integral of a mixed second derivative; this is a genuine reduction of
an unaddressed quantity to a named, exact object, not a bound. What is
established about bounding it is partial by construction: as
Remark~\ref{rem:multilayer-scope} states, $R_\times$ is not itself bounded
in this paper, and reducing that question to a chained, multi-block Jacobian
bound is as far as we take it. Proposition~\ref{prop:attnjacobian} supplies
the one closed-form ingredient this paper identified as missing for that
chain, a local attention Jacobian bound verified pointwise, without a single
violation, against Qwen2.5-1.5B-Instruct's real weights; it does not yet
chain across many layers, and we report that limitation as open rather than
implicitly resolved. Proposition~\ref{prop:curvature} supplies, separately,
the closed-form curvature constant the companion paper's own second-order
interaction bound needs and does not itself exhibit, discharging that
hypothesis on trained weights rather than leaving it assumed.

\paragraph{What remains open.} $R_\times(x)$'s own closed-form bound is the
central open item this paper leaves: Section~\ref{sec:crossterm} names the
missing ingredient precisely (chaining the per-layer attention and
normalization--MLP Jacobian bounds across every block between two touched
layers) without closing it, and Remark~\ref{rem:signlaw} adds a structural
argument, not merely an empirical failure, for why no simple sign rule for
$R_\times(x)$ should be expected to exist at all. The attention Jacobian
bound of Proposition~\ref{prop:attnjacobian} is itself verified only at a
single layer and a single perturbed token; whether the per-layer factor
compounds into something numerically vacuous after a handful of further
blocks, given the four-to-six-order-of-magnitude looseness already present
in the weight-only version of the bound at one layer, is an open empirical
question this paper does not resolve. Finally, the Qwen2.5-1.5B-Instruct
result is a genuinely mixed one on a single pretrained model and a single
mechanism (indirect object identification), not the kind of systematic
replication across models and tasks that would let us generalize past it:
one instance shows no real collapse (Anna/Mike, ratio $0.761$), the
interaction probe applies to only three of the five circuits, and none of
the three applicable instances sits in the same-block configuration the
companion paper's own theorem was proven for, so what they measure is at
most suggestive evidence about this paper's still-open multi-layer case,
not a test of a closed result.

\section{Conclusion}
\label{sec:conclusion}

A single residual block's interaction term, exact and provably second-order
bounded in the companion paper, says nothing on its own about a network with
more than one block. This paper closes part of that gap and states honestly
what remains: an exact decomposition attributes the multi-layer interaction
to same-block terms, each of them fully accounted for, plus a single
cross-layer remainder that is itself isolated exactly rather than bounded.
Closing that remainder to a usable numerical bound needed one further
ingredient, a Jacobian bound for the attention sub-block, which did not
exist in closed form before this paper and which we derive and verify,
without a single violation, against a real model's real weights; chaining
that bound across many layers is the piece we did not close, and we say so
rather than let the single-layer verification imply more than it shows.

The same real model, probed for a mechanism nobody designed into it, gives
this paper's most direct evidence that the companion paper's qualitative
predictions are not an artifact of training a small transformer to have
them. That evidence is mixed by construction: a shared carrier across
instances, collapse and dissociation present on most of five tested
instances and genuinely absent on at least one, and a nonzero interaction
measurable on three of them, at layer pairs the companion theorem was never
proven to cover. We take the mixed result as the more informative one. A
theory that only ever gets confirmed on the networks built to confirm it is
not yet a theory of anything else, and the honest next question this raises,
whether the same pattern holds across further models, further mechanisms,
and enough instances to test the sign question we could not settle here, is
sharper for having been asked on a model that was never asked to answer it.

\appendix

\section{Experimental Details and Reproducibility}
\label{app:repro}

This appendix gives what is needed to reproduce the same-block-versus-
cross-layer measurement of Section~\ref{sec:experiments}, which reuses the
companion paper's own marker-task checkpoints; the Qwen2.5-1.5B-Instruct
subsection of Section~\ref{sec:experiments} is otherwise self-describing.

\paragraph{Task.} Key and value tokens are drawn from a vocabulary of
$V=8$ symbols, with two additional marker tokens $m_A,m_B$, so the model's
vocabulary has $10$ entries. Each sequence presents $3$ key--value pairs in
random order ($6$ tokens) followed by a marker and a displayed token, giving
sequence length $8$; the target is the value of the true key, read at the
final position.

\paragraph{Architecture and training.} Each instance is a $4$-layer decoder
transformer in the Llama style: model width $d=64$, $4$ attention heads
(head dimension $16$), SwiGLU MLP with hidden width $128$, pre-normalization
by RMSNorm with $\varepsilon=10^{-6}$, learned positional embeddings, weights
in float32. There is no final normalization: the unembedding is applied
directly to the residual stream, so the logit contrast is an exactly affine
functional of $F(x)$ and Assumption~\ref{ass:readout} holds exactly, rather
than approximately, for these networks. Four instances use initialization
seed $S\in\{11,22,33,44\}$; the fifth (``inst.\ 2'') was trained earlier
under the same configuration with a different seed. No model was retrained
for the measurement reported here: it reads the same five checkpoints the
companion paper uses for its own marker-task experiments.

\paragraph{Carrier selection and ablation subspaces.} The candidate set is
every attention head and every MLP output, $4\times(4+1)=20$ sites per
instance. A site is a carrier when its mean donor-to-receiver patch recovery
$r\ge0.25$ over $8$ matched clean pairs. The companion paper's three
canonical configurations per instance are D1 (the highest-$r$ carrier in the
shallowest carrier-bearing layer, alone), DJ (all carriers in that layer),
and DJA (all carriers); together with every other non-empty subset of each
instance's carrier set, this gives $39$ configurations across the five
instances, of which $18$ touch two or more layers and are the ones this
paper's own measurement uses. For each carrier, ablation subspaces are the
leading left singular vectors of donor--receiver activation differences
collected over $32$ matched pairs, taking the smallest $k$ reaching $90\%$ of
the squared-singular-value energy, capped at $k\le4$ for a head and $k\le8$
for an MLP. The weight edit is
$W_O[:,\mathcal H]\leftarrow W_O[:,\mathcal H](I-UU^\top)$ for a head with
column block $\mathcal H$, and $W_2\leftarrow(I-UU^\top)W_2$ for an MLP:
these are the concrete instances of \eqref{eq:ablation} referred to in
Remark~\ref{rem:bridge}, with the directions estimated from data.

\paragraph{Probe seeds.} The same-block/cross-layer decomposition uses probe
seeds $20260803/20260804$ ($80$ matched pairs per configuration, restricted
to the $18$ configurations spanning two or more layers), disjoint from every
other seed used anywhere in the companion paper.

\paragraph{Code and data availability.} All experiments run against the
research repository accompanying this submission; no external data is used,
and Qwen2.5-1.5B-Instruct's weights are loaded locally rather than
downloaded during any run reported here. The relevant files are:
\begin{description}
\item[\texttt{notebook/\allowbreak verify\_crosslayer\_decomposition.jl}]\hfill\break
Proposition~\ref{prop:multilayer}, the same-block/cross-layer split on the
$18$ multi-layer marker-task configurations.
\item[\texttt{notebook/\allowbreak verify\_curvature\_bound.jl}]\hfill\break
Proposition~\ref{prop:curvature}, the finite-difference refutation test of
Remark~\ref{rem:curvature-check}.
\item[\texttt{notebook/\allowbreak verify\_attention\_jacobian\_bound.jl}]\hfill\break
Proposition~\ref{prop:attnjacobian}, the finite-difference refutation test
of Remark~\ref{rem:attnjacobian-check} on Qwen2.5-1.5B-Instruct's real
weights.
\item[\texttt{notebook/\allowbreak code4\_real\_llm\_step1\_matched\_pair\_search.jl}]\hfill\break
the search over the four candidate mechanisms and their lexical instances
that identifies IOI as the clean matched-pair family on Qwen2.5-1.5B-Instruct.
\item[\texttt{notebook/\allowbreak code4\_step2\_ioi\_patching.jl}]\hfill\break
the greedy activation-patching circuit search
(\texttt{greedy\_patch\_search!}/\texttt{backward\_prune!}) and the three
probes (collapse, dissociation, interaction) of Table~\ref{tab:qwen}.
\item[\texttt{notebook/\allowbreak generate\_code4\_step2\_figures.jl}]\hfill\break
Figures~\ref{fig:qwen-freq} and~\ref{fig:qwen-traj}, generated directly from
the JSON results of the preceding script.
\end{description}
The companion paper's own scripts train and select the carriers for the
five marker-task checkpoints this paper's same-block/cross-layer measurement
reuses. These checkpoints are stored under \texttt{notebook/marker\_ckpt/}:
\begin{description}
\item[\texttt{notebook/\allowbreak marker\_task\_experiment.jl}]\hfill\break
task, architecture, and training.
\item[\texttt{notebook/\allowbreak marker\_seed\_matrix.jl}]\hfill\break
per-seed training and carrier sweep.
\end{description}
Each script above writes a JSON results file that the corresponding figure
script reads directly, so no number in a figure is transcribed by hand.

\end{document}